\documentclass{article}

\usepackage{times}
\usepackage{graphicx} 
\usepackage{subfigure}

\usepackage{natbib}

\usepackage{algorithm}
\usepackage{algorithmic}

\usepackage{hyperref}

\author{
	Jiashi~Feng
	\\
	EECS Department \& ICSI\\
	UC Berkeley\\
	\texttt{jshfeng@berkeley.edu} \\ \\
	Huan~Xu \\
	ME Department \\
	National University of Singapore \\ 
	\texttt{mpexuh@nus.edu.sg} \\\\
	Shie~Mannor\\
	EE Department\\
	Technion\\
	\texttt{shie@ee.technion.ac.il} \\
}

\usepackage{amsthm,amsmath,amssymb}
\newtheorem{theorem}{Theorem}
\newtheorem{lemma}{Lemma}

\newtheorem{definition}{Definition}
\newtheorem{remark}{Remark}

\title{Distributed Robust  Learning}

\begin{document}
	
\maketitle

\begin{abstract}
We propose a framework for distributed robust statistical learning on {\em big contaminated data}.
The Distributed Robust Learning (DRL) framework can reduce the computational time of traditional robust learning methods by several orders of magnitude. We analyze the robustness property of DRL, showing that DRL not only preserves the robustness of the base robust learning method, but also tolerates contaminations on a constant fraction of results from computing nodes (node failures). More precisely, even in presence of the most adversarial outlier distribution over computing nodes, DRL still achieves a breakdown point of at least  $ \lambda^*/2 $, where $ \lambda^* $ is the break down point of corresponding centralized algorithm. This is in stark contrast with naive division-and-averaging implementation, which may reduce the breakdown point by a factor of $ k $ when  $ k $ computing nodes are used. We then specialize the DRL framework for two concrete cases: distributed robust principal component analysis and distributed robust regression. We demonstrate the efficiency and the robustness advantages of DRL through comprehensive simulations and predicting image tags on a  large-scale image set.

\end{abstract}


\section{Introduction}
\label{sec:introduction}

In the modern era of big data, traditional statistical learning  methods are facing two most significant challenges: (1) how to scale   current machine learning methods to the large-scale data? And (2) how to obtain accurate inference results when the data are noisy and may even contain malicious outliers? These two important challenges naturally leads to a need for  developing   \emph{scalable robust} learning methods.

In robust learning~\cite{huber2011robust}, a statistician receives $n+n_1$ samples of the form $\{\mathbf{x}_i\}_{i=1}^{n+n_1}$ for unsupervised learning or  $\{(\mathbf{x}_i,y_i)\}_{i=1}^{n+n_1}$ for supervised learning. Here $\mathbf{x}_i \in \mathbb{R}^p$ is an observation and $y_i \in \mathbb{R}$, if exists, is a real-valued response of $\mathbf{x}_i$. Among the $n+n_1$ data points, $n_1$ of them may be corrupted by gross noise or are possibly malicious outliers. The goal of robust learning is then to estimate the  parameter $\theta \in \Theta$ of interest, even though a constant fraction of outliers may exist in the data.

An example {\em par excellence} of robust learning is robust principal component analysis (RPCA)~\cite{HR-PCA,candes2011robust,feng2012robust}. RPCA aims to estimate the low-dimensional subspace fitting the inliers $\{\mathbf{x}_i\}_{i=1}^{n}$ while being resistant to the negative effect of the outliers $\{\mathbf{x}_i\}_{i=n+1}^{n+n_1}$. Another example is robust linear regression (RLR)~\cite{chen2013robust,loh2012high,chen2013noisy}, where the inlier samples $(\mathbf{x}_i,y_i)$ can be described by the linear model $y_i=\theta^\top\mathbf{x}_i$. Similar to RPCA, RLR aims to estimate the regression parameter $\theta$ without being affected by the outliers. Notice that we use the term``inlier'' to denote authentic samples generated according to the underlying statistics rule, which does not necessarily imply these samples are closer to the origin.

Traditional robust learning methods generally rely on optimizing certain robust statistics~\cite{maronna1998robust,HR-PCA} or applying some sample trimming strategies~\cite{donoho1992breakdown,feng2012robust}, whose calculations require loading all the samples into the memory or going through the data multiple times~\cite{feng2013stochastic}. Thus, the computational time of those robust learning methods is usually at least linearly dependent on the size of the sample set, $n+n_1$. For example, in RPCA~\cite{HR-PCA}, the computational time is $O((n+n_1)p^2r)$ where $r$ is the intrinsic dimension of the subspace and $ p $ is the ambient dimension. In robust linear regression~\cite{chen2013noisy}, the computational time is super-linear on the sample size: $O(p(n+n_1)\log(n+n_1))$. This rapidly increasing computation time becomes a major obstacle for applying robust learning methods to big data in practice, where the sample size easily reaches the terabyte or even petabyte scale.

In recent years, along with the rapid  increase of data, distributed learning methods become popular and necessary. Among them, one of the most popular  is simply  map-reduce~\cite{dean2008mapreduce}  (aka.\ divide-and-fusion): The data are uniformly distributed over several parallel machines and the computation results from the machines are simply fused by taking their average. Such map-reduce framework is able to shorten the computation time by several orders with negligible communication cost. However, naively implementing the robust learning algorithms in such a map-reduce framework could destroy the robustness of the algorithm. Besides the existence of outliers, distributed learning itself demands robustness as latency or breakdown of computing nodes as well as communication errors are unavoidable.

In this work, we propose a generic framework for Distributed Robust Learning (DRL) to efficiently process big data yet preserving robustness. The implementation of DRL follows the strategy similar to map-reduce~\cite{dean2008mapreduce}: DRL first distributes all the samples evenly onto $k$ machines. Then it implements a base robust learning algorithm to generate an estimate on each of the $k$ machines. Finally, it merges the $k$ individual estimates via an efficient and robust aggregation operation.  The framework is compatible with any existing robust learning methods and is able to enhance their computational efficiency with a constant factor of at least $k$ with guaranteed robustness.

A parallel implementation potentially allows significant computational speed-ups and hence the ability to cope with big data. Yet, data communication   between machines is typically slower than their processing speeds and can be a bottleneck.
Here we specify the benefits of DRL on reducing the memory usage and computation time, and show its communication cost is negligible. (1) \emph{Communication cost}. In the sample division step, samples can be directly assigned to corresponding machines without communication between different machines. In the aggregation step, the $i$th machine only needs to send its estimate $\widehat{\theta}_i$ to a specific machine. Therefore, the total communication cost of DRL is only $k\cdot s$, where $s$ is the size of $\widehat{\theta}_i$. (2) \emph{Memory cost}. Each machine needs to store samples with the size of $p (n+n_1)/k$. Compared with a single machine case, the memory cost of each machine is reduced by a factor of $k$. (3) \emph{Computation time}. Since the sample size on each machine is reduced by a factor of $k$ and the computational cost of base robust learning method is (super)-linear with respect to the sample size. The computation time on each machine may also be reduced by a factor of at least $k$.

Besides its obvious advantage of enhancing the computation efficiency for handling big data,  DRL incurs negligible robustness loss compared to the centralized robust learning methods. Suppose the breakdown point of the centralized method is $ \lambda^* $, even under the most adversarial outlier distribution, DRL still achieves a breakdown point lower bounded by $ \lambda^*/2 $. In contrast, naive division-averaging will reduce the algorithmic breakdown point\footnote{Definition of breakdown point is given in Definition \ref{def:breakdown-pt}} to $ \lambda^*/k $, when running on $ k $ machines. Thus DRL is significantly more robust, which makes it more appealing than division-averaging in practice. Though performing data permutation before division-averaging possibly helps preserve the robustness, it brings prohibitive computational overhead and thus might be not so practical.

Moreover, DRL can offer additional robustness to errors that are caused by  a machine computation, breakdown, latency or communication error. For instance, though  machines having too many outliers may individually break down, DRL is able to tolerate their bad performances and take advantage of the machines having less outliers to get better estimation. We will make this claim precise in the following sections and provide two concrete examples: distributed robust PCA and distributed robust regression.

\subsection{Related Works}
Distributed learning methods can be roughly divided into two categories: the decentralized gossip-type algorithms~\cite{boyd2006randomized} and the algorithms constructed on the map-reduce framework~\cite{dean2008mapreduce}. Gossip algorithms do not require a center node to aggregate the results, but suffer from a high communication cost.

The most relevant methods to the proposed DRL are the divide-and-conquer methods~\cite{zhang2012communication,zhang2013divide,mackey2011divide}, which are built on the map-reduce framework. In those works, similar to us, the samples are  evenly distributed on the machines and processed in parallel. However, those methods take the simple average of the estimates as the final output, which is not robust to corruption~\cite{huber2011robust}. If one machine breaks down, the final estimation can be arbitrarily bad.

Recently, several online robust learning algorithms were proposed to process the data in a sequential manner~\cite{feng2013stochastic}. Online learning methods partially mitigate the scalability issue of robust learning by reducing the memory cost for machines. However, the time complexity of those methods still depends linearly on the sample size,  hardly affordable in practice when dealing with ultra-large datasets.

\section{Preliminaries}
\label{sec:framework}

\subsection{Problem Setup}
We consider a set of $n+n_1$ observation samples $X = \{\mathbf{x}_1,\ldots,\mathbf{x}_{n+n_1}\} \subset \mathbb{R}^{p}$, which contains a mixture of $n$ authentic samples and $n_1$ outliers. Here $p$ is the dimension of the samples. The authentic samples are generated according to an underlying model (\emph{i.e.}, ground truth)  parameterized by $\theta^\star \in \Theta$, where $ \Theta $ is a  space of finite dimension. The target of statistical learning is to estimate the model parameter $ \theta^\star $ according to the provided observations.

For instance, for the problem of Principal Component Analysis (PCA), an authentic sample $ \mathbf{x} $ is   generated from a rank-$d $  matrix $\theta^\star \in\mathbb{R}^{p\times d}$ with $ d \ll p $ in the form of $\mathbf{x} = \theta^\star\mathbf{z} + \mathbf{e}$. Here $\mathbf{z}\in\mathbb{R}^{d}$ is the underlying signal and $\mathbf{e}\in\mathbb{R}^p$ denotes an additive noise to the observation. The target of PCA is to find a $ d $-dimensional  subspace of observations $ \mathbf{x} $, close to the column space of $ \theta^\star $. Another example of statistical learning is the Linear Regression (LR) problem, where the observations are pairs of covariate-response $(\mathbf{x},y)\in \mathbb{R}^p\times \mathbb{R}$. The authentic samples follow the linear model parameterized by $\theta^\star \in\mathbb{R}^p$ in the form of $y=\langle \theta^\star, \mathbf{x}\rangle + {e}$, where ${e} \in \mathbb{R}$ again denotes the additive noise. LR aims at learning a parameter $ \theta $, close to $ \theta^\star$, to best explain the observations and  provide accurate predictions for new samples.

Statistical learning with iid data is a  well understood problem and many successful learning algorithms have been developed for it. However, the existence of the outliers in the observations, which can be \emph{arbitrarily} corrupted or even \emph{maliciously} chosen in particular, makes statistical learning from these contaminated observations challenging. In this work, we focus on the  case where a constant fraction of the observations are outliers, and  we use $\lambda \triangleq {n_1}/{(n+n_1)}$ to denote this  outlier fraction throughout the paper.

\subsection{Geometric Median}
We briefly introduce the geometric median in this section, a core concept in developing the proposed distributed robust learning framework.

Geometric median, also called $ \ell_1 $-median and spatial median, is a direct generalization of the standard median proposed by~\cite{haldane1948note} and its properties have been studied in details in~\cite{kemperman1987median}. Geometric median can be defined even if the random variable does not have a finite first order moment and, most importantly, it has strong robustness properties with a breakdown of $ 0.5 $ \cite{huber2011robust}.

In particular, let $ \mathcal{H} $ be a separable Hilbert space, such as $ \mathbb{R}^d $ or $ L^2(I) $ for some closed interval $ I \subset \mathbb{R} $. We denote by $ \langle \cdot, \cdot \rangle $ its inner product and by $ \|\cdot\| $ the associated norm. The geometric median $ \widetilde{\theta} $ of a random variable $ \Theta $ taking values in $ \mathcal{H} $ is defined in~\cite{kemperman1987median}:
\begin{equation}
\widetilde{\theta} := \arg\min_{y \in \mathcal{H}} \mathbb{E}_{\Theta} \left[\|\Theta-y\|-\|\Theta\|\right].
\end{equation}
In practice, we usually consider the following empirical version of geometric median, where $ \Theta $ admits a uniform distribution on a collection of $ k $ atoms $ \theta_1,\ldots,\theta_k \in \mathcal{H} $ (which will later correspond to $ k $ individual estimations from $ k $ different machines)~\cite{minsker2013geometric}.

\begin{definition}[Geometric Median]
	\label{def:median}
	Given a finite collection of observations $\theta_1,\ldots,\theta_k$ of $\Theta$, the geometric median is the point which minimizes the total distance to all the given points, \emph{i.e.},
	\begin{equation*}
	\widetilde{\theta} = \mathrm{median}(\theta_1,\ldots,\theta_k) := \underset{y\in\mathcal{H}}{\arg\min}\sum_{j=1}^k\|y-\theta_j\|.
	\end{equation*}
\end{definition}
Geometric median (Definition \ref{def:median}) exists under rather general conditions.
Calculating the geometric median is a convex optimization problem, where any off-the-shelf convex problem solver can be employed. Due to the space limitation, we omit the optimization details.

An important property of the geometric median is that it aggregates a collection of independent estimates into a single estimate with significantly stronger concentration properties, even in presence of a constant fraction of outlying estimates in the collection.
The following lemma, straightforwardly derived from Lemma 2.1 in~\cite{minsker2013geometric}, characterizes the robustness property of the geometric median.
\begin{lemma}[Robustness of the Geometric Median]
	\label{lemma:median}
	Let $\widetilde{\theta}$ be the geometric median of the points $\theta_1,\ldots,\theta_k \in \Theta$. Fix $\gamma \in \left(0,\frac{1}{2}\right)$ and $C_\gamma = (1-\gamma)\sqrt{\frac{1}{1-2\gamma}}$. Suppose
	there exists a subset $J\subseteq \{1,\ldots,k\}$ of cardinality $|J|>(1-\gamma) k$ such that for all $j\in J$ and any point $\theta^*\in \Theta$, $\|\theta_j -\theta^*\| \leq r$. Then we have $\|\widetilde{\theta} - \theta^*\| \leq C_\gamma r$.
\end{lemma}
In words, given a set of points, their geometric median will be close to the ``true'' $\theta^*$ as long as at least half of them are close to  $\theta^*$. In particular, the geometric median will not be skewed severely even if some of the points deviate significantly away from  $\theta^*$.

\section{Distributed Robust Learning}

In this section, we present the implementation details and  main results for the proposed Distributed Robust Learning (DRL) framework. Specific examples of how to apply  DRL for concrete problems are given in the next sections.

\subsection{The DRL Framework}
The proposed DRL framework   follows a  standard division-and-conquer strategy to reduce the communication cost in distributed learning. Its core technique  is to compute the geometric median, as defined in Definition \ref{def:median},  in the fusion step to aggregate the estimates from different machines.

Suppose there are $k$ machines ready to use for distributed computation.
In the \emph{division} step, DRL evenly divides the sample set $\{\mathbf{x}_1,\ldots,\mathbf{x}_{n+n_1}\}$ into $k$ subsets. Each subset, denoted as $X_j$ of size $\lfloor (n+n_1)/k \rfloor$, is distributed onto its corresponding machine $ j $. Then these $ k $ machines run an appropriate base robust learning algorithm \emph{in parallel} for the specific problem to solve (\emph{e.g.}, robust principal component analysis or robust linear regression), to get their individual estimations of the parameter $\theta$, denoted as $\widehat{\theta}_j$. Then in the \emph{fusion} step, these $k$ estimates $\widehat{\theta}_1,\ldots,\widehat{\theta}_k$ are communicated to and aggregated on a single machine, which can be either a separate one or any machine out of the $ k $ running ones. An aggregation operation  follows to combine the $k$ estimations into  the final estimation $\widetilde{\theta}$.

Previous division-and-conquer distributed learning works~\cite{zhang2012communication,zhang2013divide,mackey2011divide} commonly propose to aggregate the $k$ separate estimations by taking their average, \emph{i.e.},
$\bar{\theta} = \frac{1}{k} \sum_{j=1}^k \widehat{\theta}_j$,
in order to reduce the variance in the estimation.
However, this average aggregation is  fragile to outlying estimations, with zero breakdown point~\cite{huber2011robust}. Specifically,  only one outlying estimate $\widehat{\theta}_j$ (which  may be caused by too many outliers in $X_j$, the breakdown or latency of machine $ j $, or  errors in communicating  $\widehat{\theta}_j$) will lead to arbitrary bad performance of the  aggregated final estimation $\bar{\theta}$. Thus, those averaging-based methods are not robust to handle sample outliers, machine breakdown or communication error, which is not unusual in practical distributed learning.

As a concrete example, suppose the outliers fractions on different machines can differ. Let the breakdown point of the base algorithm be $ \lambda^* $, \emph{i.e.}, the algorithm is able to tolerate up to $ \lambda^* $ of samples being outliers. Then in the worst case (where all outliers concentrate on a single machine), the breakdown point of the learning algorithm will be reduced to $ \lambda^*/k $ due to  non-robust aggregation operation.

Instead of averaging the estimations, DRL performs the estimates aggregation  by calculating their \emph{geometric median}, in order to take advantage of its robustness as demonstrated in Lemma \ref{lemma:median}.  Suppose we have a base robust learning algorithm, denoted as $\mathrm{RL}(X,\eta)$ with input $X$ and an algorithmic parameter $\eta$. DRL provides a distributed implementation of $\mathrm{RL}(X,\eta)$, as shown in Algorithm~\ref{alg:drl}.

\begin{algorithm}
   \caption{DRL framework}
   \label{alg:drl}
\begin{algorithmic}
   \STATE {\bfseries Input:} Sample matrix $X = [\mathbf{x}_1,\ldots,\mathbf{x}_{n+n_1}] \in \mathbb{R}^{p\times (n+n_1)}$, number of available machines $k$. Base robust learning algorithm $\mathrm{RL}(\cdot,\eta)$ with parameter $ \eta $.
   \STATE Column-wisely divide $X$ evenly into $k$ submatrices: $X=[X_1,\ldots,X_k]$.
   \STATE \textbf{Do in parallel}
   \STATE $\widehat{\theta}^{(1)} = \mathrm{RL}(X_1,\eta)$
   \STATE \hspace{8mm} $\vdots$
   \STATE $\widehat{\theta}^{(k)} = \mathrm{RL}(X_k,\eta)$
   \STATE {\bfseries Output:} $\widetilde{\theta}=\mathrm{median}\left(\widehat{\theta}^{(1)},\ldots,\widehat{\theta}^{(k)}\right)$.
\end{algorithmic}
\end{algorithm}

\subsection{Main Results: Robustness of DRL}
 We now present the robustness guarantee of DRL.
 In particular, based on the robustness property of geometric median in Lemma \ref{lemma:median}, we   obtain the following results to characterize the robustness of DRL to the corrupted estimations on a fraction of {\em machines}.
Before presenting the details, we define the following necessary quantities: for $0<\beta<\alpha<{1}/{2}$, define  their divergence as $\psi(\alpha;\beta) = (1-\alpha)\log\frac{1-\alpha}{1-\beta}+\alpha\log\frac{\alpha}{\beta}$ and let
 $ \beta^* = \beta^*(\alpha) := \max\{\beta\in(0,\alpha):\psi(\alpha,\beta)\geq 1\}.$



\begin{theorem}[Robustness Property of DRL]
\label{theo:median_robust}
Fix $\alpha \in (0,1/2)$. Assume  $\theta^\star \in \Theta$ is the ground truth parameter. Let $\widehat{\theta}_1,\ldots,\widehat{\theta}_k \in \Theta$ be a collection of independent estimations of $\theta^\star$ from $k$ machines.
Assume the estimations from $(1-\gamma)k$ machines, where $0\leq \gamma < \frac{\alpha-\beta^*}{1-\beta^*}$, satisfy $\mathbb{P}(\|\widehat{\theta}_j - \theta^\star \|>\varepsilon) \leq \beta^* < \alpha$ and the estimations from the other $\gamma k$ machines are corrupted \emph{arbitrarily}. Let $\widetilde{\theta}$ be the output of DRL. Then
\begin{equation*}
  \mathbb{P}\left(\|\widetilde{\theta} - \theta^\star \|< C_\alpha \varepsilon\right) \geq 1- \exp\left(-k(1-\gamma)\kappa_\gamma \right),
\end{equation*}
where $C_\alpha=(1-\alpha)\sqrt{\frac{1}{1-2\alpha}}$, and $ \kappa_\gamma =  \psi\left(\frac{\alpha-\gamma}{1-\gamma},\beta^*\right)$.
\end{theorem}
Basically, Theorem~\ref{theo:median_robust} states that even when $\gamma k$ machines break down (either because their outliers outnumber the breakdown point of the base robust algorithm, or due to computing node failure and communication error), DRL still guarantees that the final estimation has bounded error $C_\alpha \varepsilon$ with high probability. The proof of above theorem is straightforward  from  Lemma~\ref{lemma:median} and  Theorem 3.1 in~\cite{minsker2013geometric}.

Note that the above function  $\psi(\beta,\alpha)$ is monotonically decreasing with $\beta$ and monotonically increasing with $\alpha$. The function $C_\alpha$ is monotonically increasing in $\alpha$, which accounts for the bound relaxation after taking the geometric median.

\begin{remark}
To more explicitly appreciate the results given in Theorem~\ref{theo:median_robust}, we provide some concrete results under specific values of $\beta^*$ and $\alpha$ in Table~\ref{tab:value_p}. From these, we  observe that $C_\alpha$ decreases rapidly with decreasing $\beta^*$. When the error bound holds with failure probability less than $\beta =1\times 10^{-3}$ for $(1-\gamma)k$ machines, the value of $C_\alpha$ is very close to $1$, which means DRL indeed retains an error bound almost same as the centralized method.
\end{remark}



\begin{table}[h]
\centering
\caption{Exemplar values of $\beta^*$, $\alpha$ and $C_\alpha$ in Theorem~\ref{theo:median_robust}.}
\label{tab:value_p}
\small
\begin{tabular}{ c | c c  c c }
  \hline
  $\beta^*$ & $1\times 10^{-2}$  & $1\times 10^{-3}$  &$1\times 10^{-4}$ &$1\times 10^{-5}$ \\
  \hline
  $\alpha$ & $0.358$  & $0.22$ & $0.156$ & $0.119$\\
  \hline
  $C_\alpha$ & $1.205$   & $1.04$ &$1.018$ & $1.009$\\
  \hline
\end{tabular}
\end{table}

\begin{remark}[Trade-off between efficiency and accuracy]
In Theorem~\ref{theo:median_robust}, there is implicit dependence between the overall accuracy and  the sample size in each machine through $\beta^*$. For example, in robust PCA~\cite{HR-PCA}, $\beta^* \sim e^{-cn}$ with $n$ being the sample size. Smaller sample size may lead to larger failure probability $\beta^*$. Therefore, we need to trade-off between the efficiency (which favors increasing the number of parallel machines, $k$) and accuracy (which favors decreasing $k$ to increase the sample size on each machine) in practice.
\end{remark}

In many real-world applications, such as time series data or images and videos, the outliers may not be uniformly\footnote{Here by ``uniform'', we mean  any subset of the observations sampled at random has an identical outlier fraction
in the large sample limit.} distributed among the observations. This non-uniformness of outlier distribution, yielding different  outlier fractions on the machines,  potentially affects the final performance of distributed learning algorithms in a negative way. This robustness loss is exactly the expense of gaining efficiency by distributed computation.

To understand the potential deterioration in the  robustness  that DRL can introduce, compared with the centralized robust learning algorithms, we provide a lower bound on DRL's breakdown point~--~a widely used robustness metric for an algorithm, defined as follows \cite{huber2011robust}. 

\begin{definition}[Breakdown Point]
	\label{def:breakdown-pt}
	Breakdown point is defined as the fraction of corrupted points (outliers) that can make the output of an algorithm arbitrarily bad.
\end{definition}
Some well known breakdown point arguments include \cite{huber2011robust}: the breakdown point of empirical average is known to be zero while the median has a maximal breakdown point of $ 0.5 $.

\begin{theorem}[Breakdown Point of DRL]
	Let $ \lambda^* $ denote the breakdown point of the employed base $ \mathrm{RL} $ method in DRL. Then the breakdown point of DRL, even in the presence of adversarial outlier distribution, is always lower bounded as $ \lambda_{\mathrm{DRL}} \geq \lambda^*/2 $.
\end{theorem}
\begin{proof}
	
	Given an overall outlier fraction of $ \lambda^*/2 $, it is straightforward to see that at least $ 1/2 $ of the machines have outlier fractions no more than $\lambda^*$. Since $ \lambda^* $ is the breakdown point of the base $ \mathrm{RL} $ method, at least half of the output estimates from the machines have bounded deviation from the ground truth. Then applying Theorem \ref{theo:median_robust} guarantees that DRL, which takes the geometric median of the estimates, does not break down. Therefore, the breakdown point of DRL is at least $ \lambda^*/2 $.
	
	Then we show that this lower bound is actually tight. Let $ \epsilon, \epsilon^\prime > 0 $ be  very small positive numbers. Suppose the overall outlier fraction is $ (1/2 + \epsilon)(1+\epsilon^\prime)\lambda^* $ and outliers are distributed  among the $ k $ machines as follows: $ (1/2+\epsilon) k $ machines have  outliers fraction of $ (1+\epsilon^\prime)\lambda^* $ and $ (1/2-\epsilon) k $ machines only have inliers. In presence of such outlier distribution,  $ (1/2+\epsilon) k $ machines will break down, which then leads to breakdown of the DRL. Taking $ \epsilon,\epsilon^\prime \downarrow 0 $ provides  $ \lambda_{\mathrm{DRL}} \rightarrow \lambda^*/2 $.
	
	The above specific outlier distribution demonstrates the lower bound $ \lambda_{\mathrm{DRL}} \geq \lambda^*/2 $ is tight, and in fact it is the most adversarial distribution to DRL.
\end{proof}

	The above lower bound on its breakdown point makes DRL more appealing than simple averaging aggregation in practice. To see this, consider the case where a single machine out of $ k $ machines has all the outliers. Then only an outlier fraction of $ \lambda^*/k $ will break down the averaging. Averaging aggregation reduces the framework breakdown point to only $ \lambda_{\mathrm{Avg}}  = \lambda^*/k$~--~a severe robustness deterioration on the base learning algorithms.

\begin{remark}[Breakdown Point of DRL for Other Outlier Distributions]\label{rem.outlierdistribution}
	When outliers are uniformly distributed on the $ k $ machines, DRL preserves the breakdown point of the base learning algorithms, \emph{i.e.}, $ \lambda_{\mathrm{DRL}} =\lambda^* $. Besides, there exists a favorable outlier distribution for DRL: $ (1/2-\epsilon) k $ machines with $ 0< \epsilon < 1/2 $ only have outliers, and DRL is able to tolerate an outlier fraction up to $ (1/2 - \epsilon) + (1/2 + \epsilon)\lambda^*$. Asymptotically, $ \lim_{\epsilon \downarrow 0} \lambda_{\mathrm{DRL}} \rightarrow (1+\lambda^*)/2 $.
\end{remark}

Based on the above case studies, a natural alternative to DRL is to randomly permute the samples first and take the average of the estimates in the aggregation.	After the random permutation, outliers are uniformly distributed on the machines with high probability, and thus simply taking the average may also be  able to produce a robust estimator. However, such a strategy faces two critical problems in practice. First, in practice data may arrive directly to its respective computing nodes, and hence performing random permutation requires communicating all the data  to a central node which is often prohibitively expensive. Second, even if permutation is possible and hence non-uniform outlier distribution is no longer an issue,  the randomization-division-averaging strategy is  still fragile to machine breakdown, latency or communication error, as the averaging aggregation is not robust to such faults.

\section{Example I: Distributed Robust PCA}
In the following sections, we provide two concrete examples of DRL and their empirical evaluations  on both synthetic and real data sets, as well as comparisons with centralized and averaging-aggregation counterpart algorithms.

Classical principal component analysis (PCA) is known to be fragile to outliers and many robust PCA methods have been proposed so far (See~\cite{HR-PCA} and  references therein). However, most of those methods require to load all the data into memory and have computational cost (super-)linear in the sample size, which prevents them from being applicable for big data. In this section, we first develop a new robust PCA method which robustifies PCA via  a robust sample covariance matrix estimation, and then demonstrate how to implement it under the DRL framework to enhance the efficiency.

Given a sample matrix $X=[\mathbf{x}_1,\mathbf{x}_2,\ldots,\mathbf{x}_n]$, the standard covariance matrix is computed as $C=XX^\top$, \emph{i.e.}, $C_{ij} = \langle X_i, X_j\rangle,\forall i,j=1,\ldots,p$. Here $X_i$ denotes the $i$th row vector of matrix $X$. To obtain a robust estimate of the covariance matrix, we replace the vector inner product by a trimmed inner product, $\widehat{C}_{ij} = \langle X_i, X_j\rangle_{n_1}$, as detailed in Algorithm~\ref{alg:trimmed_ip}. Intuitively, the trimmed inner product removes the outliers having large magnitude and the remaining outliers are bounded by inliers. Thus, the obtained covariance matrix is close to the inlier sample covariance matrix.

\begin{algorithm}[h]
   \caption{Trimmed inner product $\langle\mathbf{x},\mathbf{x}' \rangle_{n_1}$}
   \label{alg:trimmed_ip}
\begin{algorithmic}
   \STATE {\bfseries Input:} Two vectors $\mathbf{x} \in \mathbb{R}^N$ and $\mathbf{x}' \in \mathbb{R}^N$, trimmed parameter $n_1$.
   \STATE Compute $q_i = \mathbf{x}_i\mathbf{x}'_i, i=1,\ldots,N$.
   \STATE Sort $\{|q_i|\}$ in ascending order and select the smallest $(N-n_1)$ ones.
   \STATE Let $\Omega$ be the set of selected indices.
   \STATE {\bfseries Output:} $h=\sum_{i\in\Omega}q_i$.
\end{algorithmic}
\end{algorithm}

After obtaining the robust estimation of covariance matrix, we perform a standard eigenvector decomposition on the covariance matrix to produce the principal component estimations. The details of the proposed base RPCA algorithm are given in Algorithm \ref{alg:rpca}.
\begin{algorithm}[h]
   \caption{Base RPCA}
   \label{alg:rpca}
\begin{algorithmic}
   \STATE {\bfseries Input:} Sample matrix $X = [\mathbf{x}_1,\ldots,\mathbf{x}_{m}] \in \mathbb{R}^{p\times m}$, outlier fraction $\lambda$, subspace dimension $d$.
   \STATE Compute the trimmed covariance matrix $\widehat{C}$ as follows (see Algorithm \ref{alg:trimmed_ip}):
   \begin{equation*}
     \widehat{C}_{ij} = \langle X^i,X^j \rangle_{\lambda m},\forall i,j=1,\ldots,p.
   \end{equation*}

   \STATE Perform eigen decomposition on $\widehat{C}$ and take the eigenvectors corresponding to the largest $d$ eigenvalues $\widehat{W}_d = [\widehat{\mathbf{w}}_1,\ldots,\widehat{\mathbf{w}}_d]$.
   \STATE {\bfseries Output:} $\widehat{W}_d$.
\end{algorithmic}
\end{algorithm}

\begin{algorithm}[h]
   \caption{DRL-RPCA}
   \label{alg:dc_rpca}
\begin{algorithmic}
   \STATE {\bfseries Input:} Sample matrix $X = [\mathbf{x}_1,\ldots,\mathbf{x}_{n+n_1}] \in \mathbb{R}^{p\times (n+n_1)}$, computing node number $k$, subspace dimension $d$, outlier fraction $\lambda = \frac{n_1}{n+n_1}$  (or set as $ 0.5 $ if unknown).
   \STATE Column-wisely divide $X$ evenly into $k$ sub-matrices: $X=[X_1,\ldots,X_k]$. Each sub-matrix has $ \lfloor (n+n_1)/k \rfloor $ column vectors.
   \STATE \textbf{Do in parallel}
   \STATE $\widehat{W}^{(1)} = \mathrm{RPCA}(X_1,d,\lambda)$
   \STATE \hspace{8mm} $\vdots$
   \STATE $\widehat{W}^{(k)} = \mathrm{RPCA}(X_k,d,\lambda)$
   \STATE \textbf{End parallel}
   \STATE Communicate $\widehat{W}^{(1)},\ldots,\widehat{W}^{(k)}$ to an aggregating machine.
   \STATE Calculate $\widehat{\mathrm{Proj}_d}^{(i)} = \widehat{W}^{(i)}\{\widehat{W}^{(i)}\}^{\top},i=1,\ldots,k$.
   \STATE {\bfseries Output:} $\widetilde{\mathrm{Proj}_d}=\mathrm{median}\left(\widehat{\mathrm{Proj}_d}^{(1)},\ldots,\widehat{\mathrm{Proj}_d}^{(k)}\right)$
\end{algorithmic}
\end{algorithm}
Plugging the proposed robust PCA method into the DRL framework gives the distributed robust PCA method, as shown in Algorithm~\ref{alg:dc_rpca}.

 We remark that in implementations, if outlier fraction $ \lambda $ is unknown, we can simply set $ \lambda=0.5 $. Moreover, in robust PCA, we cannot directly take median of the output {\em eigenvectors}. This is because the eigenvectors may rotate arbitrarily while still span the same subspace. Thus the eigenvector output can be correct but still arbitrarily far away from the ground truth.

Theorem~\ref{theo:rpca} provides a robustness guarantee for the DRL-RPCA. Due to the space limitation, the proof is provided in the supplementary material. 
\begin{theorem}[Performance Guarantee for DRL-RPCA]
\label{theo:rpca}
Suppose samples are divided onto $k$ machines, and their  outlier fractions are $\lambda_1,\ldots,\lambda_k $ respectively.
Assume the authentic samples follow sub-Guassian design with parameter $ \sigma_x^2 $.
Samples are divided onto $k$ machines. Let $\Delta_d = \sigma_d - \sigma_{d+1}$, where $\sigma_d$ denotes the $d$th largest eigenvalue of ground-truth sample covariance matrix $C$.
Let $\widetilde{\mathrm{Proj}_d}$ be the output of DRL-RPCA, and $0\leq \gamma<\frac{\alpha-\beta^*}{1-\beta^*}$. Then with a probability of at least $1- \exp\left(-k(1-\gamma)\psi\left(\frac{\alpha-\gamma}{1-\gamma},\beta^*\right)\right)$, we have
\begin{equation*}
\|\widetilde{\mathrm{Proj}_d} -\mathrm{Proj}_d^\star \|_F \leq  C_\alpha \frac{2}{\Delta_d} \frac{\lambda^\prime}{1-\lambda^\prime} \sigma_x^2 \log p,
\end{equation*}
where $ C_\alpha = (1-\alpha)\sqrt{\frac{1}{1-2\alpha}} $, and $ \lambda^\prime  $ is the $ \lfloor k (1-\gamma) \rfloor $ smallest outlier fraction in $ \{\lambda_1,\ldots,\lambda_k\} $.

\end{theorem}
\begin{proof}
	According to the proof of Theorem 4 in \cite{chen2013robust}, we have, for the covariance matrix constructed in Algorithm~\ref{alg:rpca},
	\begin{equation*}
	\|\widehat{C}-C^\star\|_\infty \leq \frac{n_1\log p}{n}\sigma_x^2 = \frac{\lambda}{1-\lambda} \sigma_x^2\log p
	\end{equation*}
	with high probability. Let $\Delta_d = \sigma_d - \sigma_{d+1}$ be the eigenvalue gap, where $\sigma_d$ denotes the $d$th largest eigenvalue of $C$.
	Then, applying the Davis-Kahan perturbation theorem \cite{davis1970rotation}, we have, whenever $\|\widehat{C}-C\|_\infty \leq \frac{1}{4}\Delta_d$,
	\begin{equation}
	\label{eqn:bounded_proj}
	\|\widehat{\mathrm{Proj}_d}-\mathrm{Proj}_d^\star\|_\infty \leq \frac{2\|\widehat{C}-C\|_\infty}{\Delta_d} \leq \frac{2}{\Delta_d} \frac{\lambda}{1-\lambda}\sigma_x^2\log p.
	\end{equation}
	
	In the DRL-RPCA algorithm, the total $(n+n_1)$ samples are divided onto $k$ machines. Suppose the outlier fraction on the machine $ i $ is  $ \lambda_i, i = 1,\ldots,k $.
	
	
	Thus, the estimated covariance matrix $\widehat{C}^{(i)}$ on machines $ i=1,\ldots,k $ satisfies
	\begin{equation*}
	\|\widehat{C}^{(i)}-C^\star\|_\infty \leq  \frac{\lambda_i}{1-\lambda_i}\sigma_x^2 \log p. 
	\end{equation*}
	
	Substituting into Eqn.~\eqref{eqn:bounded_proj}, we obtain that the estimated projection matrix is bounded as,
	\begin{equation*}
	\|\widehat{\mathrm{Proj}_d}^{(i)} -\mathrm{Proj}_d^\star\|_\infty  \leq \frac{2}{\Delta_d} \|\widehat{C}^{(i)}-C\|_\infty \leq  \frac{2}{\Delta_d} \frac{\lambda_i}{1-\lambda_i}\sigma_x^2 \log p.
	\end{equation*}
	
	Let $\widetilde{\mathrm{Proj}_d} = \mathrm{median}\left(\widehat{\mathrm{Proj}_d}^{(1)},\ldots,\widehat{\mathrm{Proj}_d}^{(k)}\right)$ as in the DRL-RPCA algorithm.
	A direct application of Theorem~\ref{theo:median_robust} gives
	\begin{equation*}
	\|\widetilde{\mathrm{Proj}_d} -\mathrm{Proj}_d^\star\|_F \leq C_\alpha \frac{2}{\Delta_d} \frac{{\lambda^\prime}}{1-{\lambda^\prime}} \sigma_x^2 \log p,
	\end{equation*}
	where $ \lambda^\prime  $ is the $ \lfloor k (1-\gamma) \rfloor $ smallest outlier fraction in $ \{\lambda_1,\ldots,\lambda_k\} $.
	
\end{proof}
Basically,  Theorem \ref{theo:rpca} says that the performance of DRL-RPCA only depends on the $ \lfloor (1-\gamma) k\rfloor $ machines with smallest fraction of outliers, and is \emph{robust} to the breakdown of the other machines. 


\section{Example II: Distributed Robust Regression}
Here, we provide an example of distributed robust regression algorithm, also under the framework of DRL. The target is to estimate the underlying linear regression model $ \theta^\star $ given the observation pairs $ \{\mathbf{x}_i,y_i\}_{i=1}^{n+n_1} $ where $ n_1 $ samples are corrupted. Similar to the above robust PCA, we adopt the robustified thresholding (RoTR) regression~\cite{chen2013robust} as a base robust regression method (see Algorithm \ref{alg:rotr}). Integrating this underlying robust regression method into the framework in Algorithm~\ref{alg:drl} gives a new distributed robust linear regression algorithm DRL-RLR, whose implementation details are provided in Algorithm \ref{alg:dc_lr}.

\begin{algorithm}[h]
	\caption{Base Robust Regression - RoTR}
	\label{alg:rotr}
	\begin{algorithmic}
		\STATE {\bfseries Input:} Covariate matrix $X = [\mathbf{x}_1,\ldots,\mathbf{x}_{n+n_1}] \in \mathbb{R}^{p\times (n+n_1)}$ and response $\mathbf{y}\in\mathbb{R}^{n+n_1}$, outlier fraction $\lambda$ (set as $ 0.5 $ if unknown).
		\STATE For $j=1,\ldots,p$, compute $\widehat{\theta}(j)=\langle \mathbf{y}, X_j \rangle_{\lambda(n+n_1)}$.
		\STATE {\bfseries Output:} $\widehat{\theta}$.
	\end{algorithmic}
\end{algorithm}

\begin{algorithm}[h]
	\caption{DRL-RLR}
	\label{alg:dc_lr}
	\begin{algorithmic}
		\STATE {\bfseries Input:} Sample pairs $\{(\mathbf{x}_1,y_1),\ldots,(\mathbf{x}_{n+n_1},y_{n+n_1})\} \subset \mathbb{R}^{p} \times \mathbb{R}$, computing node number $k$, outlier fraction $\lambda = \frac{n_1}{n+n_1}$ (or set as $ 0.5 $ if unknown).
		\STATE Divide set evenly into $k$ subsets: $X_i=\{(\mathbf{x}_{(i-1)+1},y_{(i-1)+1}),\ldots,(\mathbf{x}_{ik},y_{ik}) \}, i = 1,\ldots,k$.
		\STATE \textbf{Do in parallel}
		\STATE $\widehat{\theta}^{(1)} = \mathrm{RoTR}(X_1,\lambda)$
		\STATE \hspace{8mm} $\vdots$
		\STATE $\widehat{\theta}^{(k)} = \mathrm{RoTR}(X_k,\lambda)$
		\STATE \textbf{End parallel}
		\STATE Communicate $\widehat{\theta}^{(1)},\ldots,\widehat{\theta}^{(k)}$ to an aggregating machine.
		\STATE {\bfseries Output:} $\widetilde{\theta}=\mathrm{median}\left(\widehat{\theta}^{(1)},\ldots,\widehat{\theta}^{(k)}\right)$
	\end{algorithmic}
\end{algorithm}

We also refer the readers to~\cite{chen2013robust} for more details about the RoTR algorithm. 

Similar to the DRL-RPCA, we have the following performance guarantee for DRL-RLR. Again we  defer the proof to the supplementary material.

\begin{theorem}[Performance of DRL-RLR]
\label{theo:dist_RoTR}
Suppose samples are divided onto $k$ machines, and their  outlier fractions are $\lambda_1,\ldots,\lambda_k $ respectively.
Let $\widetilde{\theta}$ be the output of DRL-RLR. If $0\leq \gamma<\frac{\alpha-\beta^*}{1-\beta^*}$, with probability of $1- \exp\left(-k(1-\gamma)\psi\left(\frac{\alpha-\gamma}{1-\gamma},\beta^*\right)\right)$, we have
\begin{equation*}
  \|\widetilde{\theta} -\theta^\star\|_2 \leq C'_\alpha \|\theta^\star\|_2\sqrt{1+\frac{\sigma_e^2}{\|\theta^\star\|_2^2}}\left(\frac{\lambda^\prime}{1-\lambda^\prime}\sqrt{p}\log p\right)
\end{equation*}
where  $ C_\alpha = (1-\alpha)\sqrt{\frac{1}{1-2\alpha}} $, $C'_\alpha $ is an absolute constant, and $ \lambda^\prime  $ is the $ \lfloor k (1-\gamma) \rfloor $ smallest outlier fraction in $ \{\lambda_1,\ldots,\lambda_k\} $.
\end{theorem}
Also, we can see that the performance of DRL-RoTR only depends on the $ \lfloor (1-\gamma) k\rfloor $ machines with smallest fraction of outliers, and is \emph{robust} to the breakdown of the other machines.
Before proving Theorem~\ref{theo:dist_RoTR}, we first show the following performance guarantee for RoTR algorithm from \cite{chen2013robust}.
The estimation error of the RoTR is bounded as in Lemma~\ref{lemma:RoTR}.
\begin{lemma}[Performance of RoTR~\cite{chen2013robust}]
	\label{lemma:RoTR}
	Suppose the samples $ \mathbf{x} $ are from sub-Gaussian design with $\Sigma_x = I_p$, with dimension $ p $ and noise level $ \sigma_e $, then the following holds with probability at least $1-p^{-2}$. The output of RoTR satisfies the $\ell_2$ bound:
	\begin{equation*}
	\left\|\widehat{\theta}-\theta^\star \right\|_2 \leq c\|\theta^\star\|_2\sqrt{1+\frac{\sigma_e^2}{\|\theta^\star\|_2^2}}\left(\frac{\lambda}{1-\lambda}\sqrt{p}\log p\right).
	\end{equation*}
	Here $c$ is a constant independent of $p,n,\lambda$.
\end{lemma}

\begin{proof}
Based on the results in the above Lemma and Theorem~\ref{theo:median_robust}, it is straightforward to get:
\begin{equation*}
\|\widetilde{\theta} -\theta^\star\|_2 \leq C'_\alpha \|\theta^\star\|_2\sqrt{1+\frac{\sigma_e^2}{\|\theta^\star\|_2^2}}\left(\frac{\lambda'}{1-\lambda'}\sqrt{p}\log p\right)\\
\end{equation*}
where $C'_\alpha = C_\alpha c$ with $c$ being the constant in Lemma~\ref{lemma:RoTR},  $ C_\alpha = (1-\alpha)\sqrt{\frac{1}{1-2\alpha}} $, and $ \lambda^\prime  $ is the $ \lfloor k (1-\gamma) \rfloor $ smallest outlier fraction in $ \{\lambda_1,\ldots,\lambda_k\} $.
\end{proof}

\section{Simulations}
We devote this section to comparing the distributed robust learning (DRL) algorithms, including distributed RPCA and distributed robust linear regression (RLR),  with their centralized counterparts.

\begin{figure}[t!]
	\centering
	\subfigure[PCA]{
		\label{fig:drpca}
		\includegraphics[width=0.45\textwidth]{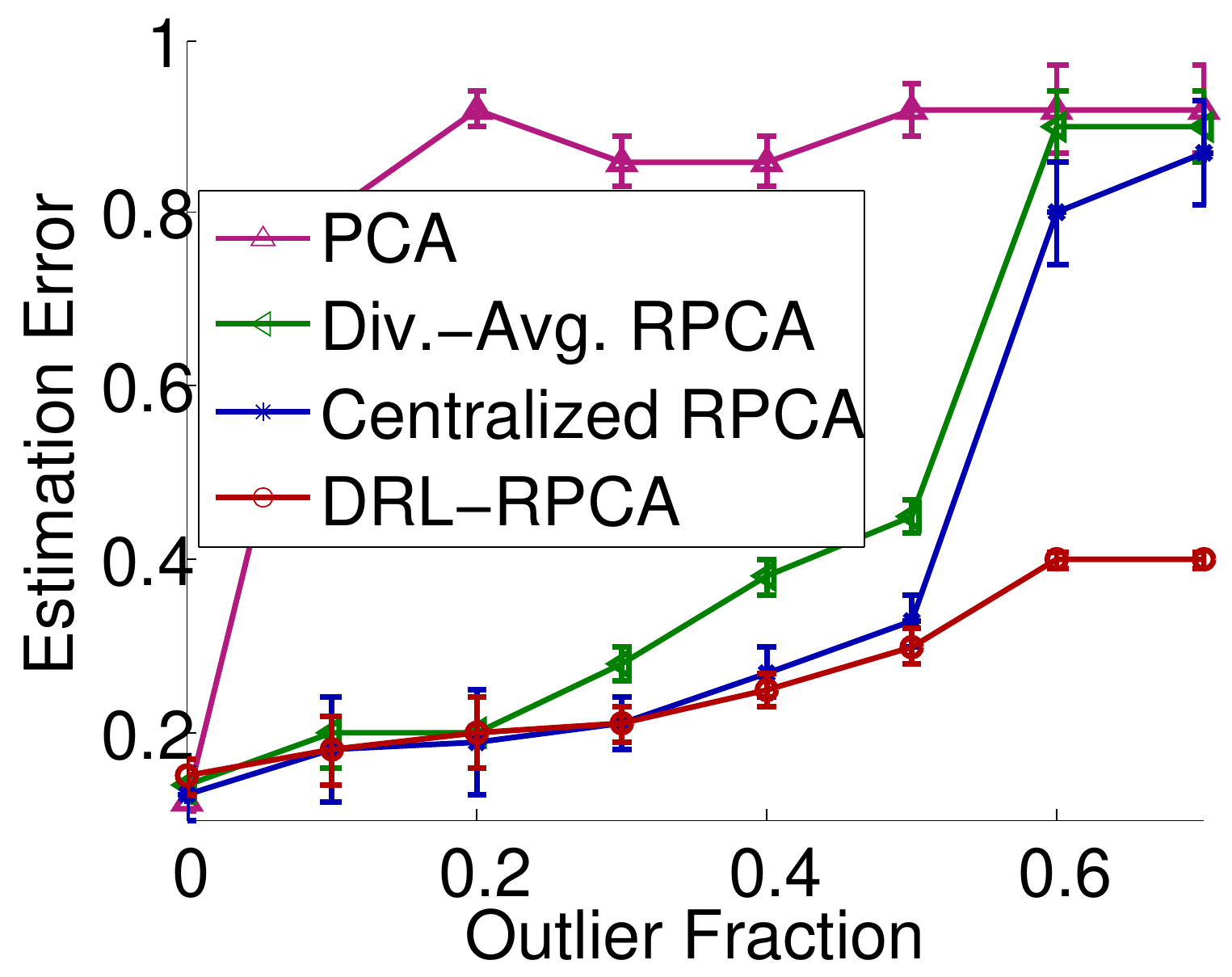}
	}
	\subfigure[LR]{
		\label{fig:drlr}
		\includegraphics[width=0.45\textwidth]{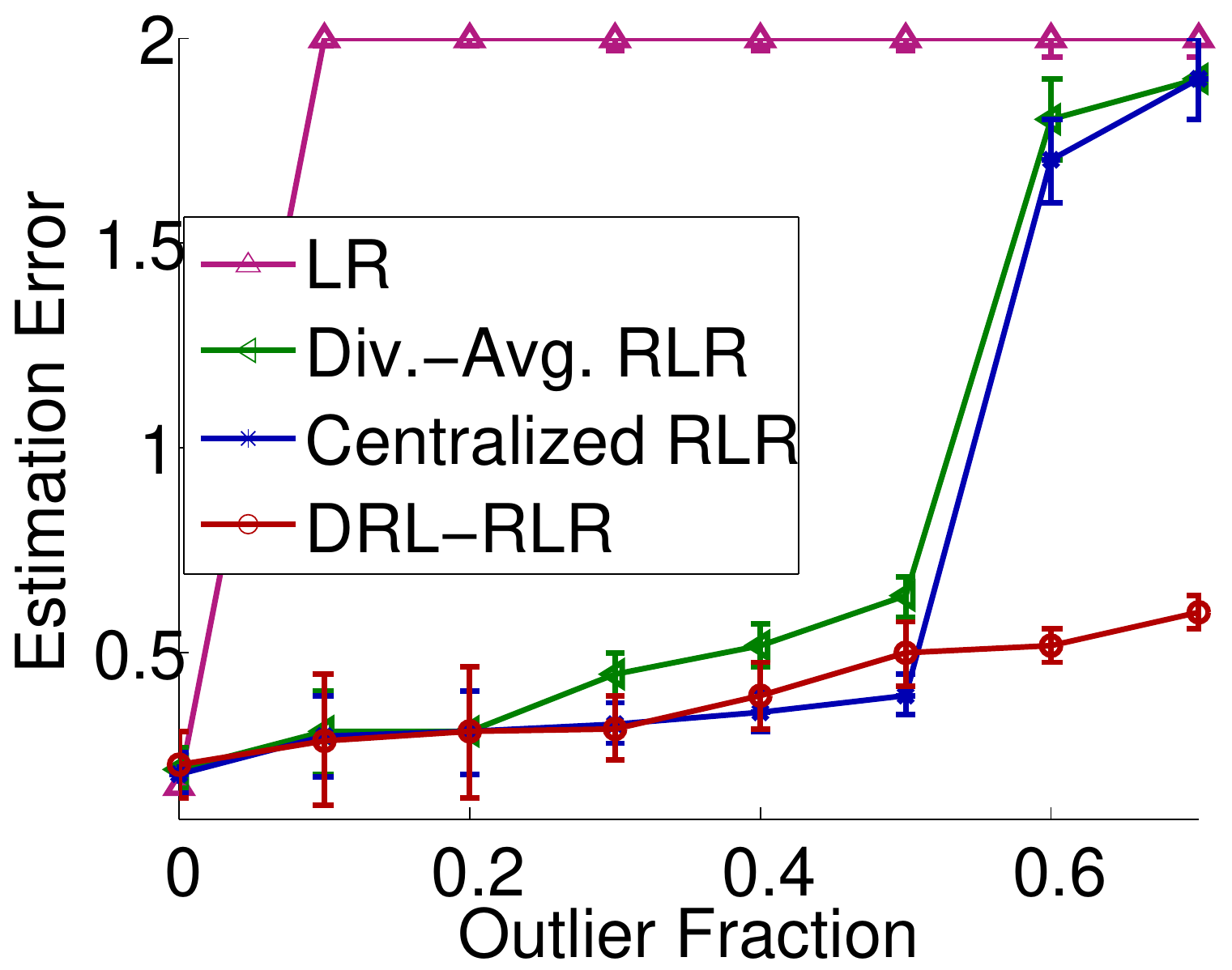}
	}
	\caption{Performance comparison between distributed, centralized and division-averaging robust  methods for (a) PCA and (b) LR, along with  standard non-robust ones. For both problems, $\sigma_e = 1$, $\sigma_o=10$, $p=100$, $n=1\times 10^{6}$ and $k=100$. For PCA,  $d=5$.
	}
	\label{fig:result}
\end{figure}

\paragraph{Synthetic data}
In  simulations of the PCA problem, samples are generated according to $\mathbf{x}_i = \theta^\star\mathbf{z}_i + \mathbf{v}_i$. Here the signal $\mathbf{z}_i \in \mathbb{R}^d$ is sampled from normal distribution: $\mathbf{z}_i \sim \mathcal{N}(0,I_d)$. The noise $\mathbf{v}_i \in \mathbb{R}^p$ is sampled as:  $\mathbf{v}_i \sim \mathcal{N}(0,\sigma_e I_p)$. The underlying matrix $\theta^\star \in \mathbb{R}^{p\times d}$ is randomly generated whose columns are then orthogonalized. The entries of outliers $\mathbf{x}_o \in \mathbb{R}^p$ are i.i.d.\ random variables from uniform distribution $[-\sigma_o,\sigma_o]$.
We use the distance between two projection matrices to measure the subspace estimation error: $\|\widetilde{\mathrm{proj}}-\mathrm{proj}^\star\|_F/\|\mathrm{proj}^\star\|_F$. Here $\widetilde{\mathrm{proj}}$ is the output estimates and $\mathrm{proj}^\star=\theta^\star {\theta^\star}^\top$.

In  simulations of the LR problem, samples $(\mathbf{x}_i,y_i)$ are generated according to $y_i = {\theta^\star}^\top \mathbf{x}_i + v_i$. Here the model parameter $\theta^\star$ is randomly sampled from $\mathcal{N}(0,I_p)$ , and $\mathbf{x}_i \in \mathbb{R}^p$ is also sampled from normal distribution: $\mathbf{x}_i \in \mathcal{N}(0,I_p)$. The noise $v_i \in \mathbb{R}$ is sampled as:  $v_i \sim \mathcal{N}(0,\sigma_e)$. The entries of outlier $\mathbf{x}_o$ are also i.i.d.\ randomly sampled from uniform distribution $[-\sigma_o,\sigma_o]$. The response of outlier is generated by $y_o = -{\theta^\star}^\top \mathbf{x}_o + v$. We use $\|\theta^\star-\widetilde{\theta}\|_2/\|\theta^\star\|_2$ to measure the error. Here $\widetilde{\theta}$ is the output estimate.

We conduct simulations with varying outlier fraction $\lambda$ from $ 0 $ to $ 0.7 $, in order to investigate the robustness of DRL with different sample contaminating degree. When $\lambda \leq 0.5$, outliers are uniformly distributed on the machines and thus the outlier fraction is around $\lambda$ on each machine. When $\lambda >0.5$, outliers are not uniformly distributed. Instead, when $\lambda=0.6$, on $\gamma = 50\%$ of the machines, the outlier fraction is $0.8$, while on the other $50\%$ machines the outlier fraction is $0.4$. Similarly, when $\lambda=0.7$, the outlier fractions are $0.9$ (on $50\%$ of the machines) and $0.5$ (on the other $50\%$ of the machines) respectively. These two adversarial cases are designed to demonstrate the additional robustness gain brought by DRL. All the simulations  are repeated for $10$ times. The average and variance of the estimation errors are plotted in Fig.~\ref{fig:drpca} and Fig.~\ref{fig:drlr} respectively.

The simulations are implemented on a PC with $2.83$GHz Quad CPU and $8$GB RAM. It takes centralized RPCA around $60$ seconds to handle $1\times 10^6$ samples with dimensionality of $100$. In contrast, distributed RPCA only costs $0.6$ seconds by using $k=100$ parallel procedures.
The communication cost here is negligible since only $100\times 5$ eigenvector matrices are communicated. For RLR simulations, we also observe about $k$ (here $k=100$) times improvement on time efficiency.

As for the performance, from Fig.~\ref{fig:drpca}, we observe that when $\lambda \leq 0.5$, DRL-RPCA, RPCA with division-averaging (Div.-Avg.\ RPCA) and centralized RPCA  achieve similar performances, which are much better than non-robust standard PCA. When $ \lambda = 0 $, \emph{i.e.}, when there are no outliers, the performances of DRL-RPCA and Div.-Avg.\ RPCA are slightly worse than standard PCA as the  quality of each mini-batch estimate deteriorates  due to the smaller sample size. However, distributed algorithms of course offer significant higher efficiency. Similar observations also hold for LR simulations from Fig.~\ref{fig:drlr}. Actually, standard PCA and LR begin to break down when $\lambda =0.1$. These results demonstrate that DRL preserves the robustness of centralized algorithms well. 

When outlier fraction $\lambda$ increases to $0.6$, centralized  (blue lines) and division-averaging algorithms (green lines) break down sharply, as the outliers outnumber their maximal breakdown point of $0.5$. In contrast, DRL-RPCA and DRL-RLR still present strong robustness and perform much better, which demonstrate that the DRL framework is indeed robust to computing nodes breaking down, and even enhances the robustness of the base robust learning methods under favorable outlier distributions across the machines.

\paragraph{Comparison with averaging} Taking the average instead of the geometric median is a natural alternative to DRL. Here we provide more simulations for the  RPCA problem  to compare these two different aggregation strategies in the presence of different errors on the computing nodes.

In distributed computation of learning problems, besides outliers, significant deterioration of the performance may result from unreliabilities, such as  latency of some machines or  communication errors. For instance, it is not uncommon that machines solve their own sub-problem at different speed, and sometimes users may require to stop the learning before all the machines output the final results. In this case, results from the slow machines are possibly not accurate enough and may hurt the quality of the aggregated solution. Similarly, communication errors may also damage the overall performance.

We simulate the machine latency by stopping the algorithms once over half of the machines finish their computation. To simulate communication error, we randomly sample $ k/10 $ estimations and flip the sign of $ 30\% $ of the elements in these estimations. The estimation errors of the solution aggregated by averaging and DRL are given in Table \ref{tab:avg-compare}. Clearly, DRL offers stronger resilience to unreliability of the computing nodes.

\begin{table}
	\caption{Comparisons of the estimation error for PCA between Division-Averaging (Div.-Avg.) and DRL, with machine latency and communication errors. Under the same parameter setting as Figure~\ref{fig:drpca}. Outlier fraction $ \lambda = 0.4 $. The average and std of the error from $ 10 $ repetitions are reported.}
	\label{tab:avg-compare}
	\centering
	\begin{tabular}{c|c|c}
		\hline
		Unreliability Type  & DRL & Div.-Avg. \\
		\hline
		\hline
		Latency & $ 0.26 \pm 0.01 $ & $ 0.42 \pm 0.01 $  \\
		\hline
		Commu. Error &$ 0.31 \pm 0.03 $  & $ 0.78 \pm 0.02  $   \\
		\hline
	\end{tabular}
\end{table}

\paragraph{Real large-scale data} We further apply the DRL-LR to solve an image classification problem on a recently released large scale image set~--~the Flickr image set\footnote{\url{http://webscope.sandbox.yahoo.com/catalog.php?datatype=i&did=67}}. This data set contains around $ 1 \times 10^8 $ images with users contributed tags. Note that this data set is actually quite noisy, in the sense of both the noisy tag annotations provided by users and the cluttered image contents. We employ linear regression to do the tag prediction over $ 200 $ tags, using the  $ 4{,}096$-dimensional deep CNN features \cite{jia2014caffe}. This large scale regression task is almost impossible for a single PC (with common $ 8 $GB memory), as simply storing the features needs around $ 50 $GB.
We randomly sample from the entire dataset  a training set of $ 0.5\times 10^8 $ images and a test set of $ 0.1\times 10^8 $ images.
In the distributed implementation, the training set is divided into $ 1{,}000 $ subsets of $ 5\times 10^5 $ images.  We compare the DRL-LR with the Division-Averaging LR to investigate the performance benefit from the robustness advantage of DRL, as well as the computation time comparison.   The results are provided in Table \ref{tab:flickr-acc}, which demonstrate  DRL-LR achieves lower  error, with a margin of $ 0.03 $, compared with Division-Averaging LR and the computation time cost due to adopted geometric median aggregations is actually negligible.

\begin{table}
	\caption{Tag prediction error  comparisons among DRL-LR, Distributed LR with division-averaging (Div.-Avg.\ LR) on the Flickr data set. The average and standard deviation of MAP's, along with computation time (in secs.), on $ 10 $ random training and test splits are reported.}
	\label{tab:flickr-acc}
	\centering
	\begin{tabular}{c|c|c}
		\hline
		  & DRL-LR & Div.-Avg. LR \\
		\hline
		\hline
		MAP & $ 0.56 \pm 0.02 $ & $ 0.59 \pm 0.01 $  \\
		\hline
		Time (secs.) & $ 3{,}002 \pm 14 $ & $ 2{,}957 \pm 5 $ \\
		\hline
	\end{tabular}
\end{table}

\section{Conclusions}

 We developed a generic  Distributed Robust Learning (DRL) framework that processes the data subsets in parallel  and aggregates  results from different subsets by taking the geometric median. DRL not only significantly enhances the time and memory efficiency of robust learning  but also preserves the robustness of the base learning algorithms. In addition, DRL was shown to bring additional resilience to  latency and breakdown of computing nodes and communication error between the nodes. Moreover, when the outliers are not uniformly distributed, the proposed framework is still robust to adversarial outliers distributions. We  provided two concrete examples, distributed robust principal component analysis and distributed robust regression, to demonstrate how DRL works.

\bibliographystyle{plain}
\bibliography{drl}

\end{document}